\pdfoutput=1
\documentclass[11pt]{article}
\usepackage{acl}
\usepackage{lmodern}
\usepackage{times}
\usepackage{latexsym}
\usepackage[T1]{fontenc}
\usepackage[utf8]{inputenc}
\usepackage{microtype}

\usepackage{graphicx}
\usepackage{subcaption}
\usepackage{booktabs}
\usepackage{algorithm}
\usepackage{algorithmic}
\usepackage{hyperref}
\usepackage{amsmath}
\usepackage{amssymb}
\usepackage{amsfonts}
\usepackage{mathtools}
\usepackage{amsthm}
\usepackage[capitalize,noabbrev]{cleveref}
\usepackage{nicefrac}
\usepackage{multirow}
\usepackage{listings}
\usepackage{pifont}
\usepackage{enumitem}
\usepackage{newunicodechar}
\newunicodechar{¬}{\ensuremath{\neg}}
\newunicodechar{∃}{\ensuremath{\exists}}
\newunicodechar{ℕ}{\ensuremath{\mathbb{N}}}
\newunicodechar{⟨}{\ensuremath{\langle}}
\newunicodechar{⟩}{\ensuremath{\rangle}}
\newunicodechar{₀}{\ensuremath{_0}}
\newunicodechar{₁}{\ensuremath{_1}}
\newunicodechar{₂}{\ensuremath{_2}}
\newunicodechar{≤}{\ensuremath{\le}}
\newunicodechar{⊢}{\ensuremath{\vdash}}

\definecolor{veritas}{HTML}{D55E00}
\definecolor{vorange}{HTML}{F97316}
\definecolor{vblue}{HTML}{2563EB}

\newcommand{\Sone}{\ensuremath{S_1}}
\newcommand{\Stwo}{\ensuremath{S_2}}

\theoremstyle{plain}
\newtheorem{theorem}{Theorem}[section]
\newtheorem{proposition}[theorem]{Proposition}

\theoremstyle{definition}

\theoremstyle{remark}

\title{VERITAS: Verifier-Guided Proof Search for\\
Zero-Shot Formal Theorem Proving}
\author{Manish Acharya$^{1}$ \quad Zhenyu Liao$^{2}$ \quad Yueke Zhang$^{1}$\\
\textbf{Kevin Leach$^{1}$ \quad Yu Huang$^{1}$ \quad Yifan Zhang$^{1}$}\\[2pt]
$^{1}$Vanderbilt University, Nashville, TN, USA \quad $^{2}$Amazon, Seattle, WA, USA\\
\texttt{manish.acharya@vanderbilt.edu} \quad \texttt{yifan.zhang.2@vanderbilt.edu}}
\hypersetup{pdftitle={VERITAS: Verifier-Guided Proof Search for Zero-Shot Formal Theorem Proving},
pdfauthor={Manish Acharya, Zhenyu Liao, Yueke Zhang, Kevin Leach, Yu Huang, Yifan Zhang}}
\begin{document}
\maketitle
\begin{abstract}
LLM-based formal provers often collapse the verifier's output into a
pass/fail bit.  We present \textbf{VERITAS}, a zero-shot framework that
feeds what Lean reports (error messages with their position, verified
partial proofs, and remaining goals) back into proof search.  Phase~1
samples $N$ whole proofs; Phase~2 searches a tree of verified partial
proofs with a Tactician, automated step repair, a Critic that ranks
states and allocates the budget, and a Retriever.  On miniF2F test,
VERITAS solves \textbf{86.9\%} with Claude Haiku~5.5 and
\textbf{82.8\%} with GPT-6 Luna, against 71.7\% and 61.1\% for Best-of-5
sampling, and solves more theorems than three re-implemented feedback
baselines at a third to a half of their search cost.  We also evaluate on
55 combinatorics problems of \textbf{CombiBench} and release code and run logs on
\href{https://github.com/CoderDoge1108/VERITAS}{GitHub}.
\end{abstract}

\section{Introduction}

Formal theorem proving, the construction of machine-checkable proofs in
interactive proof assistants such as Lean, Coq, and Isabelle, has become a
flagship benchmark for evaluating AI systems on multi-step formal thought,
because success admits a crisp definition: a candidate proof is either
accepted by the verifier or it is not. This property has driven a steady
stream of progress, pushing solve rates on miniF2F~\citep{zheng2021minif2f}
from below 30\% to above
80\%~\citep{polu2020gptf,yang2023leandojo,internlmstepprover2024,deepseek2025proverv2}.
Two complementary families drive this progress: language models that
\emph{generate} candidate
tactics~\citep{polu2020gptf,han2022pact,yang2023leandojo,azerbayev2023llemma},
and tree-search methods that \emph{explore} the combinatorial space of
partial
proofs~\citep{kocsis2006uct,silver2018alphazero,lample2022hypertree,alphaproof2024}.
Most modern systems combine the two: an LLM proposes, search disposes.

\begin{figure}[t]
  \centering
  \includegraphics[width=\columnwidth]{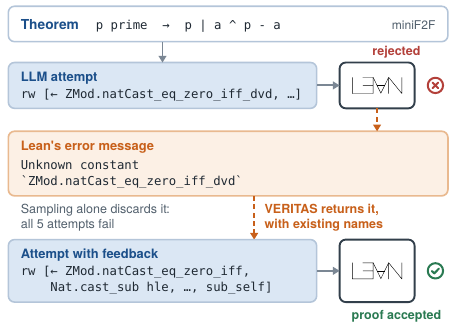}
  \caption{\textbf{Verifier feedback turns a rejected attempt into a proof} (Fermat's little theorem in miniF2F, Claude Haiku~5.5): Lean rejects a lemma name that does not exist, and the Retriever supplies one that does.}
  \label{fig:teaser}
\end{figure}

\begin{algorithm}[!t]
\caption{VERITAS Two-Phase Search}
\label{alg:veritas}
\small
\begin{algorithmic}[1]
\REQUIRE Theorems $\mathcal{T}$; budgets $N$ (Best-of-$N$), $I$
   (expansions), $K$ (candidates per expansion)
\STATE $\Sone, \Stwo \gets \emptyset$ \hfill \textit{// Phase~1: Best-of-$N$}
\FOR{$T \in \mathcal{T}$}
    \STATE $\{\pi_i\}_1^N \gets$ LLM($T$);
       $(\text{ok},\text{err})\gets\textsc{LeanBatch}(T,\{\pi_i\})$
    \IF{$\exists i\!:\!\text{ok}[i]$}
        \STATE $\Sone \gets \Sone \cup \{T\}$
    \ELSE
        \STATE $\mathcal{F}_T \gets \{(\pi_i,\text{err}_i)\}_{i=1}^N$ \hfill \textit{// rejected, with errors}
    \ENDIF
\ENDFOR
\STATE \hfill \textit{// Phase~2: search over verified partial proofs}
\FOR{$T \in \mathcal{T}\setminus\Sone$}
    \STATE Root $s_0$; \textsc{Absorb}($s_0$, $\mathcal{F}_T$);
       $I_T \gets$ \textsc{Triage}$(T,\mathcal{F}_T)$
    \FOR{$i = 1,\dots,I_T$}
        \STATE Select state $s$; get hint and lemmas for $s$
        \STATE $\{\pi_k\}_1^K \gets$ Tactician($s$, $\mathcal{F}_s$);
           $(\text{ok},\text{err})\gets\textsc{LeanBatch}$
        \STATE \textsc{Absorb}($s$, $\{(\pi_k,\text{err}_k)\}$)
        \IF{a candidate or a repaired candidate closes $T$}
            \STATE $\Stwo \gets \Stwo \cup \{T\}$; \textbf{break}
        \ENDIF
    \ENDFOR
\ENDFOR
\STATE \textbf{return} $S = \Sone \cup \Stwo$
\end{algorithmic}
\vspace{2pt}
{\footnotesize \textsc{Absorb}: repair the failing step of each rejected
proof; add verified beginnings and open-goal proofs as new states; try
closing tactics there; let the Critic score the new states; store the
rejected proofs with line and error in $\mathcal{F}_s$.}
\end{algorithm}

What is striking about this combination is how thinly the two components
\emph{communicate}. In a typical pipeline, the LM emits a batch of
candidate tactics, the verifier checks each one, and the search algorithm
updates value estimates from the resulting pass-or-fail outcomes. Anything
richer that the verifier produces along the way (a parse error pointing at
a malformed token, a type mismatch naming the offending hypothesis, a
partial goal that has been simplified but not closed) is summarised into a
scalar reward or thrown away entirely. Yet these intermediate outputs are
exactly the artefacts a human prover attends to
(Figure~\ref{fig:teaser}): collapsing them into a
single bit forces the search to relearn what the verifier has already told
it explicitly. Recent rich-feedback methods in
RL~\citep{huebotter2026sdpo} make the same observation in the post-training
regime; we explore the dual at inference time.

We present \textbf{VERITAS}, a zero-shot framework built around a single
design commitment: \emph{what the verifier reports should reach every
stage of proof search}.  Phase~1 is a Best-of-$N$ sweep.  Phase~2 runs
only on the theorems Phase~1 leaves and searches a tree whose nodes are
\emph{verified partial proofs}: when Lean rejects a proof at some line,
the steps before that line are kept and the goal Lean prints there
becomes a new state.  Four agents act on this tree.  A \emph{Tactician}
(an LLM) writes ways to finish the proof from a state and sees the
attempts already rejected there, the line where each failed, and Lean's
message.  A \emph{Critic} (an LLM) ranks new states and, after Phase~1,
decides whether a theorem receives the full search budget.  A BM25
\emph{Retriever} over Mathlib offers lemmas for the current goal and
existing names close to any name Lean reports as unknown.  A rule-based
\emph{Strategist} adds a strategy hint.  Around them, automation tries
closing tactics at every new state and repairs the single step of a
proof that Lean rejected.  Every theorem solved in Phase~1 is kept, so
the proofs added in Phase~2 can be attributed to the search
(\cref{prop:monotonicity}).

On the miniF2F test set (244 theorems), VERITAS solves 86.9\% with
Claude Haiku~5.5 and 82.8\% with GPT-6 Luna, where its Phase~1, a
Best-of-5 sweep, solves 71.7\% and 61.1\%.  We compare with three
feedback methods re-implemented in the same harness with the same
models: a stepwise agent with backtracking in the style of
COPRA~\citep{thakur2024copra}, self-reflection in the style of
Reflexion~\citep{shinn2023reflexion}, and sketching with automated hole
filling in the style of DSP+~\citep{cao2025dspplus}.  VERITAS solves more
theorems than each of them with both models while issuing a fraction of
their LLM calls: its search phase costs \$0.77--1.36 of API usage per
run against \$1.87--3.47 for the baselines.  Ablations locate the gain
in iteration (8--13 theorems beyond the first expansion) and in Lean's
error messages (4--5 theorems), while the Critic's triage cuts the
search cost by 21--29\%.

\begin{figure*}[t]
  \centering
  \includegraphics[width=0.98\linewidth]{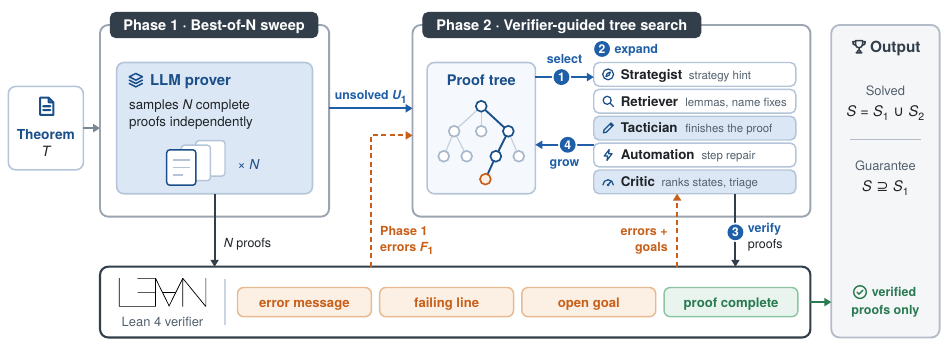}
  \caption{\textbf{VERITAS two-phase protocol.}  Phase~1 samples $N$ proofs; unsolved theorems $U_1$ go to Phase~2, which (1)~selects a state, (2)~expands it (shaded: LLM agents), (3)~verifies in one batched call, and (4)~grows the tree with verified partial proofs.  Dashed orange: what Lean reports.}
  \label{fig:architecture}
\end{figure*}

\paragraph{Contributions.}
(i)~A zero-shot framework in which every agent acts on what Lean
reports: error messages with their position, verified partial proofs,
and remaining goals;
(ii)~a two-phase Best-of-$N$$\;\rightarrow\;$tree-search protocol with an
internal monotonicity guarantee that isolates the contribution of the
search;
(iii)~batched Lean validation, which checks all candidates of an
expansion in one Lean call and returns per-candidate error positions;
(iv)~an evaluation on miniF2F with two models, three re-implemented
feedback baselines and ablations, with all run logs released; and
(v)~results on 55 combinatorics problems of the harder
\textbf{CombiBench}~\citep{liu2025combibench}.

\section{The VERITAS Framework}
\label{sec:framework}

VERITAS has two phases over one Lean verifier
(Figure~\ref{fig:architecture}; Algorithm~\ref{alg:veritas}).  Phase~1
samples whole proofs.  Phase~2 is a search in which four agents
(Strategist, Retriever, Tactician, Critic) and a fixed automation step
share a tree of verified partial proofs.

\paragraph{What Lean returns.}
For every candidate proof the harness extracts four things from Lean's
output: whether the proof is \emph{accepted}; otherwise the \emph{error
messages}; the \emph{line} of the candidate each message refers to; and,
when a proof runs without error but leaves goals open, the
\emph{remaining goal} as Lean prints it.  The last case is not treated
as a failure: the candidate is a verified partial proof.

\paragraph{States.}
A state $s=(\mathcal{T},\tau_{1:t},g,\mathcal{F}_s)$ consists of the
theorem $\mathcal{T}$, a tactic prefix $\tau_{1:t}$ that Lean has
checked, the goal $g$ Lean reports after it, and the attempts
$\mathcal{F}_s$ rejected from this state, each with its failing line and
Lean's message.  The root has an empty prefix.

\paragraph{Agents.}
The \emph{Strategist} maps the goal's syntax to one of six strategies
(e.g.\ \texttt{induction}, \texttt{cases}, \texttt{rewrite}) by fixed
rules and passes it on as a hint.  The \emph{Retriever} ranks the
176{,}873 theorem and lemma declarations of Mathlib with BM25 against
the current goal; when an error in $\mathcal{F}_s$ reports an unknown
constant or identifier, it also returns the existing declarations
closest to that name (Appendix~\ref{app:retriever}).  The
\emph{Tactician} receives the theorem, the verified prefix, the goal,
$\mathcal{F}_s$, the retrieved lemmas and the hint, and writes $K$
different tactic blocks that finish the proof from $g$
(Appendix~\ref{app:prompts}).  The \emph{Critic} is called in two
places: it assigns each new state the probability that it can be
completed in a few steps, and after Phase~1 it triages the theorem
(\cref{sec:two_phase}).

\subsection{Search over Verified Partial Proofs}
\label{sec:mcts}

Phase~2 repeats four steps until the theorem is proved or the budget is
spent: it selects a state of the tree, expands it with new candidate
proofs, lets automation repair what Lean rejects, and grows the tree
with what Lean accepts.

\emph{Selection.}  Each state carries a value $V(s)\in[0,1]$, set by the
Critic ($0.5$ for the root).  Starting at the root, the search compares
staying at the current state $s$, scored
$V(s)+c\sqrt{\ln(n_s{+}1)/(1{+}e_s)}$, with moving to a child $s'$,
scored $V(s')+c\sqrt{\ln(n_s{+}1)/(1{+}n_{s'})}$, where $n$ counts
visits and $e_s$ the expansions already made at $s$.  It descends while
some child scores higher and otherwise expands $s$, so a state can be
asked for further candidates more than once.

\emph{Expansion and verification.}  The Tactician's $K$ candidates are
appended to the prefix of $s$ and checked in one batched Lean call
(\cref{sec:batched}).  An accepted candidate ends the search.

\emph{Automation.}  For a rejected candidate, the step at Lean's first
error is replaced by a fixed combination of closing tactics
(\texttt{linarith}, \texttt{nlinarith}, \texttt{omega},
\texttt{norm\_num}, \texttt{simp}, \texttt{grind}, \ldots;
Appendix~\ref{app:batched}) while the rest of the proof is kept, and the
repaired candidates are checked in one further call.  A repaired proof
that Lean accepts is a solution; one that now fails later has a longer
verified beginning.  The same closing tactics are tried at every new
state.

\emph{Growth.}  The search does not discard a rejected proof.  The
top-level steps before the failing line, once Lean confirms them, form a
new state whose goal is the one Lean reports; a candidate that runs
without error but leaves goals open becomes a state in the same way.
The Critic scores the new states in a single call, states it considers
infeasible are dropped, and the rejected candidates join
$\mathcal{F}_s$.

\subsection{Two-Phase Design and Monotonicity Guarantee}
\label{sec:two_phase}

Most theorems that an LLM can prove at all are proved by a handful of
samples, so VERITAS spends the search only where sampling fails; the
split also makes the contribution of the search measurable.

\textbf{Phase~1 (Best-of-$N$ sweep)} asks the LLM for $N$ complete proof
attempts and validates them in one batched Lean call; if any succeeds,
it is returned.  \textbf{Phase~2} activates only if Phase~1 fails.  It
first re-reads the rejected Phase~1 proofs with their error positions,
so that their verified beginnings become the first states of the tree
and their errors the root's $\mathcal{F}$, before any further LLM call.
The Critic then triages the theorem: given the statement and the first
attempts with Lean's verdicts, it estimates the probability that the
search will succeed, and theorems below a threshold receive a small
expansion budget instead of the full one.

\begin{proposition}[Monotonicity]
\label{prop:monotonicity}
\textit{For any fixed VERITAS run on problem set $\mathcal{T}$ and
Phase~1 budget $N$, the final VERITAS solve set is a superset of that
run's own Phase~1 solve set: $S(\mathcal{T}) \supseteq \Sone(\mathcal{T})$.}
\end{proposition}
\begin{proof}
If VERITAS's Phase~1 solves a theorem, that proof is immediately kept
and Phase~2 is never invoked for that theorem.  If Phase~1 fails,
Phase~2 may or may not succeed, but it cannot remove a theorem already
in $\Sone$.
\end{proof}

\noindent The guarantee is what makes the attribution in
\cref{sec:experiments} exact: with a shared Phase~1, every difference
between VERITAS, its ablations and the baselines arises on the theorems
Phase~1 leaves.

\subsection{Batched Lean Validation}
\label{sec:batched}

Rather than invoking Lean separately for each of $K$ candidates, we
compile them into one \texttt{.lean} file as private theorems with
distinct names, each consisting of the statement, the verified prefix
and one candidate, and run a single \texttt{lake env lean --json} call.
Lean's messages carry positions, so each one is mapped back to a
candidate and to a line inside it.  This reduces validation cost from
$O(K)$ to $O(1)$ Lean calls per expansion, because Lean's per-file
overhead (importing Mathlib, elaborating the shared statement) is paid
once rather than $K$ times (Appendix~\ref{app:batched}).

\subsection{A Proof Walkthrough}
\label{sec:walkthrough}

Figure~\ref{fig:walkthrough} follows a miniF2F theorem through VERITAS,
using the attempts and Lean messages logged in our run with Claude
Haiku~5.5.  All five Phase~1 proofs are rejected at the same step, a
bound that \texttt{nlinarith} cannot prove.  Lean's message gives the
line and the goal there, so the five steps before it are kept as a
verified partial proof.  The first expansion proves the bound from that
state by a different argument and finishes the proof.

\begin{figure}[!tb]
\centering
\setlength{\fboxsep}{5pt}
\fcolorbox{black!35}{white}{\begin{minipage}{0.94\columnwidth}
\small\raggedright
\textbf{Given} $n\in\mathbb{N}$, $10\le n$, $\exists x,\ x^2=n$ and
$\exists t,\ t^3=n$.\\
\textbf{Goal:} $64 \le n$.

\smallskip\hrule\smallskip
\textbf{Phase~1: rejected}\hfill error and line kept\\
{\ttfamily\scriptsize by\_contra h; push\_neg at h\\
obtain ⟨x, hx⟩ := h₁\\
obtain ⟨t, ht⟩ := h₂\\
have hx8 : x < 8 := by nlinarith\\
have ht4 : t < 4 := by nlinarith\\
interval\_cases x <;> interval\_cases t <;> \ldots}\\
Lean, at \texttt{ht4}: \emph{linarith failed to find a contradiction}\\
{\ttfamily\scriptsize \ldots\ ht : t \^{} 3 = n, a : 4 ≤ t ⊢ False}\\
All five attempts fail at this step.

\smallskip\hrule\smallskip
\textbf{Verified partial proof}\hfill becomes a state\\
The steps up to \texttt{hx8} are accepted.  Lean reports the goal\\
{\ttfamily\scriptsize h : n < 64, hx : x \^{} 2 = n,\\
ht : t \^{} 3 = n, hx8 : x < 8 ⊢ False}

\smallskip\hrule\smallskip
\textbf{Phase~2: finished from that state}\hfill accepted\\
{\ttfamily\scriptsize have ht4 : t < 4 := by\\
\hspace*{1em}by\_contra h4; push\_neg at h4\\
\hspace*{1em}have := Nat.pow\_le\_pow\_left h4 3\\
\hspace*{1em}norm\_num at this; linarith\\
interval\_cases t <;> interval\_cases x <;>\\
\hspace*{1em}norm\_num at hx ht <;> omega}\\
The Tactician sees the rejected step and Lean's message, and proves the
bound through monotonicity of the cube.
\end{minipage}}
\caption{\textbf{Proof walkthrough} for miniF2F's \texttt{mathd\_numbertheory\_5} (logged; abbreviated with \ldots).}
\label{fig:walkthrough}
\end{figure}

\section{Experiments}
\label{sec:experiments}

We organise the evaluation around four research questions and an
ablation study.  \textbf{RQ1}: does VERITAS prove more theorems than
sampling and than other ways of using verifier feedback?
\textbf{RQ2}: what does it cost?  \textbf{RQ3}: how does Phase~2 find
its proofs?  \textbf{RQ4}: can the Critic tell which theorems are worth
the search budget?  The \textbf{ablation study} removes one component at
a time.

\begin{table*}[!t]
\centering
\small
\setlength{\tabcolsep}{0pt}
\begin{tabular*}{\textwidth}{@{\extracolsep{\fill}}l rrrrrrr rrrrrrr@{}}
\toprule
 & \multicolumn{7}{c}{\textbf{Claude Haiku 5.5}} & \multicolumn{7}{c}{\textbf{GPT-6 Luna}} \\
\cmidrule{2-8} \cmidrule{9-15}
Method & Solved & \% & $\Delta$ & Ph.\,2 & \% & Lean & Cost & Solved & \% & $\Delta$ & Ph.\,2 & \% & Lean & Cost \\
\midrule
Best-of-1            & 136 & 55.7 & $-16.0$ & 0  & 0.0  & 1.0 & 0.00 & 118 & 48.4 & $-12.7$ & 0  & 0.0  & 1.0 & 0.00 \\
Best-of-5 (Phase~1)  & 175 & 71.7 & 0.0     & 0  & 0.0  & 1.0 & 0.00 & 149 & 61.1 & 0.0     & 0  & 0.0  & 1.0 & 0.00 \\
DSP+-style           & 177 & 72.5 & $+0.8$  & 2  & 2.9  & 2.7 & 1.87 & 153 & 62.7 & $+1.6$  & 4  & 4.2  & 3.7 & 2.69 \\
Reflexion-style      & 205 & 84.0 & $+12.3$ & 30 & 43.5 & 2.1 & 2.48 & 186 & 76.2 & $+15.2$ & 37 & 38.9 & 3.0 & 2.85 \\
COPRA-style          & 209 & 85.7 & $+13.9$ & 34 & 49.3 & 6.8 & 2.32 & 193 & 79.1 & $+18.0$ & 44 & 46.3 & 9.3 & 3.47 \\
\textbf{VERITAS}     & \textbf{212} & \textbf{86.9} & $\mathbf{+15.2}$ & \textbf{37} & \textbf{53.6} & 4.2 & \textbf{0.77} & \textbf{202} & \textbf{82.8} & $\mathbf{+21.7}$ & \textbf{53} & \textbf{55.8} & 5.4 & \textbf{1.36} \\
\bottomrule
\end{tabular*}
\caption{\textbf{miniF2F test (244 theorems)}, one run per method.  $\Delta$: points over Best-of-5.  \emph{Ph.\,2}: theorems solved among those Phase~1 leaves (69 and 95), with their share.  \emph{Lean}: batched Lean calls per theorem.  \emph{Cost}: API cost of the search phase (USD).}
\label{tab:main}
\end{table*}

\subsection{Setup}

All methods share the benchmark, the models, the Lean environment and
the first sampling sweep, so that differences between them can be
attributed to how they use the verifier.

\paragraph{Benchmark.}
\textbf{miniF2F}~\citep{zheng2021minif2f}: competition mathematics
(AIME, AMC, IMO, MATH) in Lean~4 with Mathlib.  All design choices were
made on the validation split (243 theorems); the test split (244
theorems) was run once at the end, and all reported numbers are test
numbers.

\paragraph{Models.}
Claude Haiku~5.5 and GPT-6 Luna, each used for every LLM role, without
any training or fine-tuning.

\paragraph{Methods.}
\textbf{Best-of-$N$} samples $N$ proofs and checks them in one batched
Lean call.  \textbf{VERITAS} runs Phase~1 with $N{=}5$ and Phase~2 with
up to $I{=}10$ expansions of $K{=}6$ candidates
(Appendix~\ref{app:hyperparameters}).  The three baselines
are our re-implementations in the same harness with the same model, not
the original systems, and start from the same Phase~1 results as
VERITAS.  \textbf{COPRA-style}: a stepwise agent that proposes one
tactic per query from the current goal and the tactics that failed in
this state, with backtracking (up to 30 queries).
\textbf{Reflexion-style}: six rounds of five whole proofs, each round
preceded by a short self-written reflection on the earlier failures and
their Lean errors.  \textbf{DSP+-style}: six rounds in which the model
drafts subgoals and writes sketches whose open steps are closed by a
fixed combination of automated tactics; the model sees no verifier
feedback.  All methods have 600\,s per theorem.

\paragraph{Protocol.}
All proofs are verified with \texttt{lake env lean --json} against
Lean~4~\citep{demoura2021lean4} and a pinned
Mathlib~\citep{mathlib2020}.  For each model one Phase~1 sweep is shared
by all methods, so they differ only on the theorems it leaves (69 for
Claude Haiku~5.5, 95 for GPT-6 Luna).  We report one run per method.  A
theorem that a method leaves unsolved because it reaches the time limit
is retried once with the same settings; the rule is the same for every
method and all counts and costs include these retries.  Differences are
tested with a paired bootstrap over theorems (Appendix~\ref{app:stats});
Appendix~\ref{app:runs} gives the counts before the retries and two
further VERITAS runs per model.

\subsection{RQ1: Does VERITAS Prove More Theorems?}
\label{sec:rq1}

We first compare the number of theorems proved, then ask whether the
gain holds across models, problem sources and individual theorems.

\begin{figure}[!tb]
  \centering
  \includegraphics[width=0.99\columnwidth]{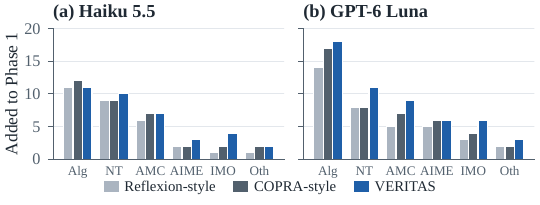}
  \caption{\textbf{RQ1:} theorems each feedback method adds to Phase~1, per source category.}
  \label{fig:rq1}
\end{figure}

\emph{Overall.}  VERITAS solves \textbf{86.9\%} of the test set with
Claude Haiku~5.5 and \textbf{82.8\%} with GPT-6 Luna
(Table~\ref{tab:main}).  Phase~2 adds 37 and 53 theorems to the 175 and
149 its Best-of-5 sweep solves, i.e.\ $+15.2$ and $+21.7$ points, and
more than half of the theorems sampling leaves.  VERITAS has the highest
count with both models.  With GPT-6 Luna it is 9 theorems ahead of the
COPRA-style agent and 16 ahead of the Reflexion-style baseline
($p{=}0.004$ and $p{<}0.001$, paired bootstrap).  With Claude Haiku~5.5
it is 7 ahead of the Reflexion-style baseline ($p{=}0.02$) and 3 ahead
of the COPRA-style agent, a lead within noise ($p{=}0.3$) that VERITAS
reaches at a third of that agent's cost (\cref{sec:cost}).  The
DSP+-style baseline, which gives the model no verifier feedback, stays
within four theorems of Phase~1.  The result is stable: over three runs
VERITAS solves 212, 214 and 213 theorems with Claude Haiku~5.5 and 202,
203 and 204 with GPT-6 Luna, above every baseline each time; 34 of the
69 and 49 of the 95 Phase~2 theorems are solved in all three runs, and
only 8 and 11 change between runs.

\emph{Across models.}  The search helps the weaker sampler more.
GPT-6 Luna is 10.6 points behind Claude Haiku~5.5 after Best-of-5
sampling and 4.1 points behind after Phase~2, and it gains more from the
search than from going from one sample to five ($+21.7$ against $+12.7$
points).  Verifier feedback thus narrows the gap between models as well
as the gap to a proof.

\emph{By category.}  Figure~\ref{fig:rq1} shows where the additional
proofs are.  The gap to the baselines is widest on the hardest sources.
VERITAS adds four IMO problems to the eight of Phase~1 with Claude
Haiku~5.5 and six to the four with GPT-6 Luna, where the baselines add
at most two and four.  With GPT-6 Luna it also adds 11 number-theory and
9 AMC problems against at most 8 and 7 (counts in
Appendix~\ref{app:percat}).

\emph{Which theorems.}  VERITAS solves 5 (Haiku) and 10 (Luna) theorems
that the COPRA-style agent misses and misses 2 and 1 that it solves;
against the Reflexion-style baseline the counts are 8 and 18 against 1
and 2.  Nine theorems with GPT-6 Luna and three with Claude Haiku~5.5
are solved by VERITAS and by no baseline (Appendix~\ref{app:only}).

\subsection{RQ2: What Does It Cost?}
\label{sec:cost}
\label{sec:efficiency}

A search is only useful if its proofs are affordable.  We measure the
logged API usage of every run, its running time, and the number of
Lean calls it needs.

\begin{table}[t]
\centering
\footnotesize
\setlength{\tabcolsep}{2.6pt}
\begin{tabular}{l r r r r r}
\toprule
Run & Calls & In & Out & USD & s/thm \\
\midrule
\multicolumn{6}{l}{\emph{Claude Haiku 5.5}} \\
Best-of-1                  & 346     & 0.09 & 0.78 & 0.40 & -- \\
Best-of-5 (Phase~1)        & 445     & 0.11 & 1.32 & 0.67 & -- \\
VERITAS, Phases 1+2        & 847     & 1.20 & 2.64 & 1.44 & -- \\
VERITAS, Phase 2           & 402     & 1.09 & 1.32 & 0.77 & 187 \\
\quad one expansion        & 184     & 0.36 & 0.41 & 0.24 & \phantom{0}84 \\
COPRA-style                & 3{,}285 & 4.23 & 3.80 & 2.32 & 359 \\
Reflexion-style            & 1{,}340 & 3.14 & 4.33 & 2.48 & 285 \\
DSP+-style                 & 1{,}539 & 0.61 & 3.62 & 1.87 & 243 \\
\midrule
\multicolumn{6}{l}{\emph{GPT-6 Luna}} \\
Best-of-1                  & 399     & 0.07 & 1.15 & 0.58 & -- \\
Best-of-5 (Phase~1)        & 479     & 0.08 & 1.76 & 0.89 & -- \\
VERITAS, Phases 1+2        & 1{,}015 & 1.00 & 4.30 & 2.25 & -- \\
VERITAS, Phase 2           & 536     & 0.92 & 2.54 & 1.36 & 299 \\
\quad one expansion        & 233     & 0.27 & 0.78 & 0.42 & 116 \\
COPRA-style                & 4{,}893 & 3.39 & 6.27 & 3.47 & 791 \\
Reflexion-style            & 1{,}716 & 2.33 & 5.23 & 2.85 & 570 \\
DSP+-style                 & 2{,}333 & 0.56 & 5.27 & 2.69 & 576 \\
\bottomrule
\end{tabular}
\caption{\textbf{Logged usage per run}: LLM calls, input/output tokens (millions), API cost (USD), seconds per Phase-2 theorem.  Indented and baseline rows cover the search phase only.}
\label{tab:cost_breakdown}
\end{table}

\begin{figure}[!tb]
  \centering
  \includegraphics[width=0.99\columnwidth]{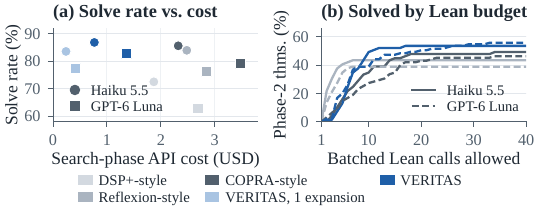}
  \caption{\textbf{RQ2:} \emph{(a)}~solve rate against search-phase API cost; \emph{(b)}~share of Phase-2 theorems solved within a given number of Lean calls.}
  \label{fig:results}
\end{figure}

\emph{API cost.}  The search is confined to the theorems the one-call
sweep leaves, and within them VERITAS is the cheapest feedback method we
ran (Table~\ref{tab:cost_breakdown}; Figure~\ref{fig:results}a).  Its
search phase costs \$0.77 (Claude Haiku~5.5) and \$1.36 (GPT-6 Luna)
against \$1.87--3.47 for the baselines: two to three cents per
additional proof against seven to eight for the COPRA- and
Reflexion-style baselines (both models are priced at \$0.10 and \$0.50
per million input and output tokens).  The Tactician is queried per
expansion rather than per tactic, which gives 400--540 LLM calls per run
against 1{,}300--4{,}900.  Even with a single expansion, at
\$0.24--0.42, VERITAS solves about as many theorems as the
Reflexion-style baseline at \$2.5--2.9.  A complete run over the
benchmark, including Phase~1, costs \$1.44 and \$2.25, that is, less
than one cent per theorem, of which the search accounts for 53\% and
60\%.  In Phase~2 the prompts, which carry the rejected attempts and
Lean's messages, make up 45\% (Haiku) and 27\% (Luna) of the tokens.

\emph{Time and verifier budget.}  VERITAS is also faster on the same
theorems: 187\,s per theorem against 243--359\,s with Claude Haiku~5.5
and 299\,s against 570--791\,s with GPT-6 Luna.
Figure~\ref{fig:results}(b) shows how it uses the verifier.  VERITAS
spends its first Lean calls on closing tactics, re-reading and repair,
which cost no LLM tokens; from about ten calls on it is ahead of both
baselines and keeps adding proofs where they level off.  Over all 244
theorems it uses 4.2--5.4 batched Lean calls per theorem against
6.8--9.3 for the stepwise agent.

\subsection{RQ3: How Does Phase~2 Find Its Proofs?}
\label{sec:inside}

\begin{table}[t]
\centering
\footnotesize
\setlength{\tabcolsep}{2.5pt}
\begin{tabular}{@{}l rr rr@{}}
\toprule
 & \multicolumn{2}{c}{Haiku 5.5} & \multicolumn{2}{c}{GPT-6 Luna} \\
\cmidrule(lr){2-3} \cmidrule(lr){4-5}
Proof found & Tact. & Autom. & Tact. & Autom. \\
\midrule
re-reading Phase~1 (no LLM) & \phantom{0}0 & 10 & \phantom{0}0 & 17 \\
\quad from a non-root state      & \phantom{0}0 & \phantom{0}4 & \phantom{0}0 & \phantom{0}1 \\
in the first expansion           & 13  & \phantom{0}2 & 18 & \phantom{0}2 \\
\quad from a non-root state      & \phantom{0}6 & \phantom{0}0 & 11 & \phantom{0}1 \\
in a later expansion             & 12 & \phantom{0}0 & 10 & \phantom{0}6 \\
\quad from a non-root state      & \phantom{0}6 & \phantom{0}0 & \phantom{0}8 & \phantom{0}2 \\
\midrule
Total                            & 25 & 12 & 28 & 25 \\
\bottomrule
\end{tabular}
\caption{\textbf{How the Phase~2 proofs are found}: stage, and whether the Tactician or the automation completes the proof (indented rows are subsets).}
\label{tab:closing}
\end{table}

\begin{figure}[!tb]
  \centering
  \includegraphics[width=0.99\columnwidth]{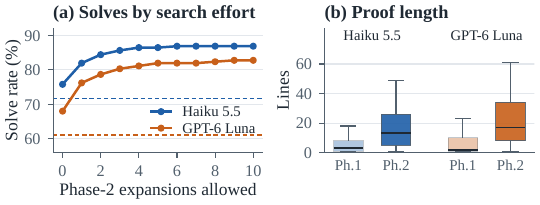}
  \caption{\textbf{RQ3:} \emph{(a)}~solve rate by number of expansions allowed (dashed: Phase~1); \emph{(b)}~lines of the accepted proofs.}
  \label{fig:analysis}
\end{figure}

Table~\ref{tab:closing} traces the 37 and 53 proofs found in Phase~2.

\emph{What Lean reports.}  The rejected Phase~1 proofs already say why
they fail.  On 46 of the 69 Phase~2 theorems of Claude Haiku~5.5 and 45
of the 95 of GPT-6 Luna, an attempt stops at an arithmetic step that
\texttt{linarith} or \texttt{omega} cannot close; on 41 and 56, Lean
accepts every step and reports the goals that remain; 13 and 18 have a
type mismatch, 14 and 11 a lemma name that does not exist
(Figure~\ref{fig:teaser}; Appendix~\ref{app:examples}), and 8 and 7 an exceeded recursion depth.
Each message points at one step and one goal.

\emph{Before any new LLM call.}  Ten (Haiku) and 17 (Luna) proofs are
found while the rejected Phase~1 proofs are re-read: a sampled proof
that fails at one step is completed by step repair, or closing tactics
succeed at a state it reached (Figure~\ref{fig:example}).  A quarter to
a third of the Phase~2 gain thus comes from proofs that sampling had
almost found and that a pass/fail check discards.  Overall, automation
completes 12 and 25 of the proofs, 7 and 21 of them by repairing a
single step.

\emph{From verified partial proofs.}  16 of the 37 proofs with Claude
Haiku~5.5 and 23 of the 53 with GPT-6 Luna are completed from a state
other than the root: the accepted proof begins with steps of an earlier,
rejected attempt and continues from the goal Lean reported there
(Figure~\ref{fig:walkthrough}).

\emph{Few rounds, small trees.}  68\% and 70\% of the proofs are found
by the end of the first expansion (Figure~\ref{fig:analysis}a).  A
solved theorem takes 1.4 and 1.5 expansions on average, an unsolved one
3.8 and 3.6; the trees hold 5.2 and 3.1 states on average and are at
most four levels deep.  Phase~2 behaves less like a deep search than
like a few rounds of informed rewriting, each starting from what Lean
has confirmed.

\emph{Longer proofs.}  The proofs differ in kind from those of Phase~1
(Figure~\ref{fig:analysis}b): the median accepted proof has 3 (Haiku)
and 2 (Luna) lines in Phase~1 and 13 and 17 lines in Phase~2, with up to
64 lines.

\begin{figure}[!tb]
\begin{lstlisting}[emptylines=1,frame=single,framesep=3pt,rulecolor=\color{black!35},xleftmargin=4pt,xrightmargin=4pt]
theorem amc12b_2021_p3 (x : ℝ)
  (h₀ : 2 + 1 / (1 + 1 / (2 + 2 / (3 + x)))
        = 144 / 53) : x = 3 / 4

-- Phase 1 attempt: rejected at its last step
field_simp at h₀
linarith
-- Lean: linarith failed to find a contradiction

-- Step repair, no LLM call: accepted
field_simp at h₀
first | done | rfl | omega | linarith | positivity
  | (norm_num; done) | nlinarith | grind | …
\end{lstlisting}
\caption{\textbf{Step repair} (Claude Haiku~5.5): the automation replaces the failing last step by closing tactics, and Lean accepts.}
\label{fig:example}
\end{figure}

\subsection{RQ4: Can the Critic Tell Which Theorems Are Worth the Budget?}
\label{sec:triage}

\begin{table}[t]
\centering
\footnotesize
\setlength{\tabcolsep}{4pt}
\begin{tabular}{l rr rr}
\toprule
 & \multicolumn{2}{c}{Haiku 5.5} & \multicolumn{2}{c}{GPT-6 Luna} \\
\cmidrule(lr){2-3} \cmidrule(lr){4-5}
Critic's estimate $\hat p$ & $n$ & solved & $n$ & solved \\
\midrule
$\hat p < 0.1$ (small budget) & 23 & \phantom{0}2 & 29 & \phantom{0}1 \\
$0.1 \le \hat p < 0.3$        & 11 & \phantom{0}5 & 14 & \phantom{0}7 \\
$0.3 \le \hat p < 0.6$        & 11 & \phantom{0}9 & 17 & 14 \\
$\hat p \ge 0.6$              & 13 & 11 & 18 & 14 \\
\bottomrule
\end{tabular}
\caption{\textbf{Triage estimate and outcome} for the theorems still open after Phase~1 is re-read.}
\label{tab:triage}
\end{table}

\begin{figure}[!tb]
  \centering
  \includegraphics[width=0.99\columnwidth]{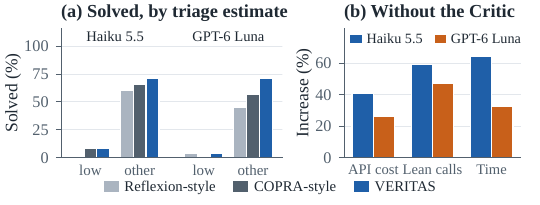}
  \caption{\textbf{RQ4:} \emph{(a)}~share solved by each method on theorems rated $\hat p<0.1$ (\emph{low}) and on the others; \emph{(b)}~increase in the search phase without the Critic.}
  \label{fig:rq4}
\end{figure}

After Phase~1 the Critic estimates, from the statement and the rejected
attempts, the probability $\hat p$ that the search will succeed;
theorems with $\hat p<0.1$ receive two expansions instead of ten.

\emph{The estimate is informative.}  Of the theorems rated below 0.1,
two of 23 (Haiku) and one of 29 (Luna) are solved, against 25 of 35 and
35 of 49 of those rated higher (Table~\ref{tab:triage}).

\emph{The low-rated theorems are hard for every method.}  The small
budget could hide proofs that a longer search would find, but the
baselines, which spend their full budget on these theorems, solve at
most two of them (Figure~\ref{fig:rq4}a; COPRA-style: 2 and 0;
Reflexion-style: 0 and 1).  Several miniF2F statements are false or
trivial as formalized~\citep{cao2025dspplus}, and the Critic marks 9 and
12 theorems as false as stated, of which one is solved.

\emph{What the Critic saves.}  Without the Critic the search phase costs
41\% and 26\% more, makes 59\% and 47\% more Lean calls and takes 64\%
and 32\% longer (Figure~\ref{fig:rq4}b).  The price is small: without triage the
search solves one more theorem with Claude Haiku~5.5 and five more with
GPT-6 Luna.

\subsection{Ablation Study}
\label{sec:ablations}

\begin{table}[t]
\centering
\footnotesize
\setlength{\tabcolsep}{2.1pt}
\begin{tabular}{@{}l rrrr rrrr@{}}
\toprule
 & \multicolumn{4}{c}{Haiku 5.5} & \multicolumn{4}{c}{GPT-6 Luna} \\
\cmidrule(lr){2-5} \cmidrule(lr){6-9}
Variant & Slv. & $\Delta$ & Lean & Cost & Slv. & $\Delta$ & Lean & Cost \\
\midrule
\textbf{VERITAS}          & \textbf{212} & $0$ & 4.2 & 0.77 & \textbf{202} & $0$ & 5.4 & 1.36 \\
$-$ Strategist            & 213 & $+1$  & 4.0 & 0.78 & 203 & $+1$  & 5.0 & 1.20 \\
$-$ Retriever             & 211 & $-1$  & 4.1 & 0.74 & 201 & $-1$  & 5.5 & 1.37 \\
$-$ automation            & 211 & $-1$  & 2.5 & 0.73 & 201 & $-1$  & 3.6 & 1.66 \\
$-$ tree                  & 210 & $-2$  & 3.4 & 0.86 & 204 & $+2$  & 4.6 & 1.50 \\
$-$ Critic                & 213 & $+1$  & 6.2 & 1.08 & 207 & $+5$  & 7.7 & 1.72 \\
$-$ error feedb.          & 208 & $-4$  & 4.1 & 0.63 & 197 & $-5$  & 6.2 & 1.54 \\
One expansion             & 204 & $-8$  & 2.7 & 0.24 & 189 & $-13$ & 3.1 & 0.42 \\
Phase~1 only              & 175 & $-37$ & 1.0 & 0.00 & 149 & $-53$ & 1.0 & 0.00 \\
\bottomrule
\end{tabular}
\caption{\textbf{Ablations on miniF2F test} (\emph{Slv.}: solved; \emph{Lean}: calls per theorem; \emph{Cost}: search phase, USD).}
\label{tab:ablations}
\end{table}

\begin{figure}[!tb]
  \centering
  \includegraphics[width=0.99\columnwidth]{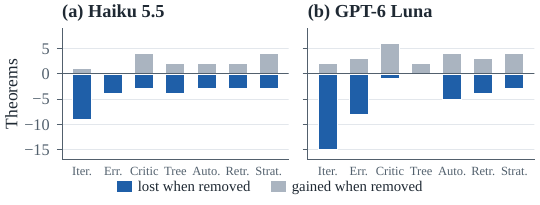}
  \caption{\textbf{Which theorems change} when one part is removed: solved only by the full system (lost) and only by the reduced one (gained).  \emph{Iter.}: one expansion only; \emph{Err.}: no error feedback.}
  \label{fig:ablation}
\end{figure}

Table~\ref{tab:ablations} removes one part of VERITAS at a time on the
test split with Phase~1 fixed, and Figure~\ref{fig:ablation} shows which
theorems each removal changes.

\emph{The search phase.}  Phase~2 is responsible for 37 and 53 of the
solved theorems, 17\% and 26\% of all VERITAS proofs.

\emph{Iteration.}  A single expansion already recovers 29 and 40 of them
at a third of the cost.  The remaining rounds add a further 8 and 13
theorems ($p{=}0.008$ and $p{=}0.001$): continuing from what Lean has
confirmed is the most valuable use of the budget.

\emph{Lean's messages.}  Hiding the rejected attempts and their errors
from the Tactician costs four theorems with Claude Haiku~5.5
($p{=}0.03$) and five with GPT-6 Luna, where the search then also needs
more Lean calls and 13\% more API cost to solve fewer theorems.

\emph{Robustness.}  Removing the Strategist, the Retriever, the
automation or the tree changes the count by at most two theorems on
either model: no component is a single point of failure, because the
agents offer overlapping routes to a proof.  Figure~\ref{fig:ablation}
shows that the components are nevertheless at work.  Each removal loses
up to five theorems and gains up to four others: a step that automation
repairs can also be rewritten by the Tactician one expansion later, and
a state reached through a verified beginning can also be reached from
the root.

\emph{Efficiency.}  What the components buy is efficiency.  The Critic
lowers the search cost by 29\% and 21\% (\cref{sec:triage}).  The tree
lowers it by about 10\% on both models.  Automation shifts work from the
LLM to the verifier: without it the search makes 40\% and 33\% fewer
Lean calls, and with GPT-6 Luna pays 22\% more for the LLM.

\section{Discussion}
\label{sec:discussion}

VERITAS treats the verifier as a source of information rather than a
filter.  Iterating on what Lean reports is what pays: continuing for several
rounds adds 8--13 theorems over a single round, and Lean's messages add
further theorems with both models.  The
comparison with the baselines suggests that how the feedback is used
matters as much as whether it is used: the COPRA-style agent sees Lean's
errors at every step but spends one model call per tactic, and the
Reflexion-style baseline replaces the raw messages by the model's own
summary of them; both solve fewer theorems at two to three times the
cost.  The idea extends beyond Lean: SMT unsat cores, Coq and Agda
goals, symbolic-execution counterexamples and model-checking traces have
similar structure.  Planning is the natural next step: the theorems that
remain need long chains of intermediate facts, and a Critic that already
recognises hopeless theorems could learn to steer the search towards
proofs.

\paragraph{What remains.}
Of the 244 test theorems, 28 (Claude Haiku~5.5) and 38 (GPT-6 Luna) are
solved by none of the four feedback methods, most of them AMC (11 and
14) and IMO (6 and 8) problems.  Many of the others are within reach:
counting every run we made, including the ablations and the repeated
runs, only 25 and 28 theorems are never solved, so 219 and 216 theorems
(89.8\% and 88.5\%) are proved at least once.  A search that combines
several configurations, or simply runs longer, should therefore gain
further theorems without a stronger model.

\paragraph{CombiBench.}
We also evaluate on the 55 proof-based problems of
CombiBench~\citep{liu2025combibench}, from the IMO and its shortlist~(22),
Brualdi's \emph{Introductory Combinatorics}~(29) and other
competitions~(4) (Appendix~\ref{app:combibench}).  It is much harder than
miniF2F: with Claude Haiku~5.5, Best-of-5 sampling solves 10 of the 55
problems, all of them textbook exercises, and VERITAS 17, which is
70\% more than sampling.  Phase~2 thus adds seven theorems where sampling
finds few, and 38 problems remain open, which makes the benchmark a
stress test for feedback-driven provers.

\section{Related Work}
\label{sec:related}

VERITAS draws on three lines of work: provers built on language models,
proof search that uses the verifier without training, and learning from
feedback richer than a final reward.


\paragraph{LLM-based theorem proving.}
GPT-f~\citep{polu2020gptf}, PACT~\citep{han2022pact}, and
ReProver~\citep{yang2023leandojo} established tactic generation and
retrieval.  Later systems, including InternLM-StepProver~\citep{internlmstepprover2024},
Goedel-Prover~\citep{lin2025goedel}, DeepSeek-Prover-V2~\citep{deepseek2025proverv2},
Llemma~\citep{azerbayev2023llemma}, LEGO-Prover~\citep{xia2024legoProver},
and Baldur~\citep{first2023baldur}, explore fine-tuning, pretraining,
lemma libraries, and repair.  They are not direct zero-shot baselines, but
they motivate the question VERITAS asks at inference time: how much of
the proof assistant's output can be used without training?

\paragraph{Search and verifier feedback without training.}
Tree search for proving builds on UCT~\citep{auer2002ucb,kocsis2006uct}
and AlphaZero~\citep{silver2018alphazero}; HyperTree Proof
Search~\citep{lample2022hypertree}, Thor~\citep{jiang2022thor}, and
DeepSeek-Prover-V1.5~\citep{xin2025dsproverv15} combine it with trained
models.  Among inference-only methods, COPRA~\citep{thakur2024copra}
runs a stepwise agent with backtracking over Lean's responses,
Draft-Sketch-Prove~\citep{jiang2023dsp} and DSP+~\citep{cao2025dspplus}
let automated provers close the gaps of a model-written sketch, and
Reflexion~\citep{shinn2023reflexion} has the model reflect on failed
trials.  VERITAS differs in what it keeps from a failed attempt: the
error with its position goes back to generation, the verified beginning
becomes a search state with Lean's own goal, and the failing step is
handed to automation.  We compare with re-implementations of the three
approaches in \cref{sec:rq1}.

\paragraph{Process supervision and rich feedback.}
Step-level supervision improves math reasoning~\citep{lightman2023verify,uesato2022solving}.
Feedback-based approaches also support code repair and self-correction~\citep{zhang2025dpofplus,zhang2026synthfix,zhang2025codegrad}.
SDPO~\citep{huebotter2026sdpo} distills such feedback during training;
VERITAS is the inference-time dual with deterministic Lean feedback.  Its
roles align with multi-agent reasoning~\citep{hong2023metagpt,wu2024autogen,li2023camel,zhang2026spiral}.

\section{Conclusion}

VERITAS shows that what Lean reports can guide zero-shot proof search
without any training.  Phase~1 keeps the easy proofs; Phase~2 turns
rejected attempts into error feedback, verified partial proofs and
repairable steps, reaching 86.9\% and 82.8\% on miniF2F test with two
inexpensive general-purpose models, ahead of three re-implemented
feedback baselines at a third to a half of their search cost.  Iterating
on Lean's feedback carries the gain, and the Critic's triage keeps the
search cheap.  The search helps the weaker model more, narrowing the gap, and
over 40\% of its proofs continue from steps that Lean had already
confirmed in a rejected attempt.  Planning longer arguments, a learned Critic, and
distilling Phase~2 traces into the model are natural next steps.

\section*{Limitations}

Our conclusions are bounded by the scope of the evaluation and by the
compute the method needs.

\textbf{Scope.}  The evaluation covers one benchmark, miniF2F, and two
proprietary models accessed through their APIs, so exact reproduction
depends on model availability; several miniF2F statements are false or
trivial as formalized, which bounds every method.  The baselines are our
re-implementations with the same models and the same Phase~1, not the
original systems, which use other models, provers and budgets
(Appendix~\ref{app:finetuned_context}).  We report one run per method,
with one retry of theorems that reach the time limit
(Appendix~\ref{app:runs}); with Claude Haiku~5.5 the lead over the
COPRA-style agent is three theorems and within noise.  The triage
threshold was chosen on the validation split.  CombiBench was
evaluated with one model.

\textbf{Compute.}  VERITAS uses more inference compute than sampling:
4.2--5.4 batched Lean calls per theorem instead of one, and three to
five minutes per theorem that enters Phase~2.

\section*{Ethical Considerations}
The task concerns formal mathematics and produces proofs checked by
Lean~4. Generated candidates require kernel verification before use.
The experiments use Lean, Mathlib, and miniF2F for formal theorem
proving. CombiBench~\citep{liu2025combibench} contains theorem statements
from public competitions and a textbook; we use the statements of its
public MIT-licensed release and credit its authors. API-based
inference incurs financial and computational costs, reported in
\cref{sec:cost}.

\section*{AI Assistance}
AI assistants supported the implementation and debugging of the code,
the running of experiments and the analysis of their results, as well as
writing, figure preparation, LaTeX formatting, and paper editing. The
authors are responsible for the scientific claims, experiments,
references, and final submitted content.

\bibliography{references}

\appendix

\section{Impact and Reproducibility}
This work advances automated formal reasoning; the main risk is
over-trusting unchecked generated proofs.  Code, prompts,
hyperparameters, the CombiBench problem list, and the per-theorem results of
every reported run are available at
\url{https://github.com/CoderDoge1108/VERITAS}; a script rebuilds
every table from the result files.  All reported successes are
Lean~4/Mathlib verified.

\section{Broader Theorem-Proving Context}
\label{app:finetuned_context}

\Cref{tab:finetuned_context} summarizes published miniF2F results.
Model training, underlying language models, prover versions, and
inference budgets differ across these systems, so the numbers are
context rather than a controlled comparison.

\begin{table}[ht]
\centering
\scriptsize
\setlength{\tabcolsep}{2.5pt}
\begin{tabular}{@{}l c@{}}
\toprule
Method (regime) & miniF2F (\%) \\
\midrule
\multicolumn{2}{l}{\emph{No training, inference-only}} \\
\quad\textbf{VERITAS} (this work, Claude Haiku~5.5)  & \textbf{86.9} \\
\quad\textbf{VERITAS} (this work, GPT-6 Luna)  & \textbf{82.8} \\
\quad DSP+~\citep{cao2025dspplus}                        & 80.7 \\
\quad COPRA-GPT-4~\citep{thakur2024copra}                & 30.7 \\
\midrule
\multicolumn{2}{l}{\emph{Fine-tuned on Lean traces}} \\
\quad ReProver~\citep{yang2023leandojo}                  & 26.5 \\
\quad InternLM-StepProver~\citep{internlmstepprover2024} & 65.9 \\
\quad DeepSeek-Prover-V2~\citep{deepseek2025proverv2}           & 88.9 \\
\bottomrule
\end{tabular}
\caption{\textbf{Published miniF2F results.} Solve rates
reported by prior proof systems under their own protocols.}
\label{tab:finetuned_context}
\end{table}

\section{Per-Category and Paired Results}
\label{app:percat}

\Cref{tab:percat} gives the counts behind Figure~\ref{fig:rq1}, and
\cref{tab:overlap} pairs VERITAS with each baseline on the theorems
Phase~1 leaves.

\begin{table}[ht]
\centering
\footnotesize
\setlength{\tabcolsep}{2.2pt}
\begin{tabular}{@{}l rrrrrr@{}}
\toprule
 & Alg. & NT & AMC & AIME & IMO & Oth. \\
       & {\tiny $n{=}88$} & {\tiny $67$} & {\tiny $45$} & {\tiny $15$} & {\tiny $19$} & {\tiny $10$} \\
\midrule
\multicolumn{7}{l}{\emph{Claude Haiku 5.5}} \\
Best-of-5 (Phase~1)  & 72 & 54 & 25 & 11 & \phantom{0}8 & 5 \\
DSP+-style           & 72 & 54 & 26 & 11 & \phantom{0}9 & 5 \\
Reflexion-style      & 83 & 63 & 31 & 13 & \phantom{0}9 & 6 \\
COPRA-style          & \textbf{84} & 63 & \textbf{32} & 13 & 10 & \textbf{7} \\
\textbf{VERITAS}     & 83 & \textbf{64} & \textbf{32} & \textbf{14} & \textbf{12} & \textbf{7} \\
\midrule
\multicolumn{7}{l}{\emph{GPT-6 Luna}} \\
Best-of-5 (Phase~1)  & 64 & 52 & 20 & \phantom{0}6 & \phantom{0}4 & 3 \\
DSP+-style           & 65 & 53 & 21 & \phantom{0}6 & \phantom{0}5 & 3 \\
Reflexion-style      & 78 & 60 & 25 & 11 & \phantom{0}7 & 5 \\
COPRA-style          & 81 & 60 & 27 & \textbf{12} & \phantom{0}8 & 5 \\
\textbf{VERITAS}     & \textbf{82} & \textbf{63} & \textbf{29} & \textbf{12} & \textbf{10} & \textbf{6} \\
\bottomrule
\end{tabular}
\caption{\textbf{Solved test theorems per source category} (counts).}
\label{tab:percat}
\end{table}

\begin{table}[ht]
\centering
\footnotesize
\setlength{\tabcolsep}{3.2pt}
\begin{tabular}{l rrrr rrrr}
\toprule
 & \multicolumn{4}{c}{Haiku 5.5 ($n{=}69$)} & \multicolumn{4}{c}{GPT-6 Luna ($n{=}95$)} \\
\cmidrule(lr){2-5} \cmidrule(lr){6-9}
VERITAS vs. & both & V & B & none & both & V & B & none \\
\midrule
COPRA-style     & 32 & \phantom{0}5 & 2 & 30 & 43 & 10 & 1 & 41 \\
Reflexion-style & 29 & \phantom{0}8 & 1 & 31 & 35 & 18 & 2 & 40 \\
DSP+-style      & \phantom{0}1 & 36 & 1 & 31 & \phantom{0}2 & 51 & 2 & 40 \\
\bottomrule
\end{tabular}
\caption{\textbf{Paired outcomes on Phase-2 theorems}: solved by both, only by VERITAS (V), only by the baseline (B), by neither.}
\label{tab:overlap}
\end{table}

\section{Implementation Details and Hyperparameters}
\label{app:hyperparameters}

\Cref{tab:hyperparams} lists the settings of the evaluated system.  They
were fixed on the validation split before the test split was run.  Both
APIs ignore sampling parameters for these models, so the models' default
sampling is used.

\begin{table}[ht]
\centering
\scriptsize
\setlength{\tabcolsep}{3pt}
\begin{tabular}{@{}l l l@{}}
\toprule
Component & Parameter & Value \\
\midrule
Phase~1     & $N$ (Best-of-$N$)   & 5 \\
\midrule
Search      & expansions $I$      & 10 \\
            & candidates $K$      & 6 \\
            & exploration $c$     & 0.7 \\
            & max.\ depth         & 8 \\
            & new states per expansion & $\le 5$ \\
            & budget              & 600\,s per theorem \\
\midrule
Tactician   & rejected attempts shown & $\le 8$ per state \\
            &                     & (14 lines, 400 chars of error) \\
\midrule
Critic      & state dropped if value & $< 0.05$ \\
            & triage threshold    & 0.1 \\
            & triage budget       & 2 expansions \\
\midrule
Automation  & repair rounds       & 2 ($\le 8$ proofs per round) \\
\midrule
Retriever   & corpus              & 176{,}873 Mathlib decls. \\
            & scoring             & BM25 ($k_1{=}1.2$, $b{=}0.75$) \\
            & lemmas for the goal & 4 \\
            & names per unknown name & 3 (for $\le 3$ names) \\
\midrule
Strategist  & type                & rule-based, 6 strategies \\
\midrule
Lean        & call                & \texttt{lake env lean --json} \\
            & time limit per call & 90\,s \\
\bottomrule
\end{tabular}
\caption{\textbf{Settings of the evaluated system.}}
\label{tab:hyperparams}
\end{table}

\paragraph{Software.}
We use Lean \texttt{v4.29.0} and Mathlib at commit \texttt{6eec119},
frozen for all experiments.

\paragraph{Harness.}
Four details of proof checking change results and should be kept by a
re-implementation: proofs are checked with their indentation intact;
statements that use the retired notation $\sum$\,\texttt{x in s} are rewritten
to $\sum$\,\texttt{x} $\in$ \texttt{s} (32 test statements do not parse otherwise); a
leading \texttt{by} written by the model is dropped, since the harness
supplies it; and when one candidate of a batch exceeds the time limit,
the batch is split rather than discarded.

\paragraph{Hardware and time.}
The LLMs run behind their providers' APIs and Lean runs locally on a CPU
server.  On the theorems that enter Phase~2, VERITAS takes 187\,s
(Claude Haiku~5.5) and 299\,s (GPT-6 Luna) per theorem on average; the
COPRA-style baseline takes 359\,s and 791\,s.  Runs use three parallel
workers, except the reported VERITAS runs (six and eight).

\section{Prompts}
\label{app:prompts}

The prompts of the evaluated system are reproduced below; placeholders
in braces are filled at run time.

\subsection{Phase~1 Sampler}
\label{app:phase1_prompt}

\begin{lstlisting}
Prove the following Lean 4 theorem. Suggest {k} different
proof attempts.

```lean
{theorem}
```

Output ONLY a JSON array of tactic strings. Each string can
be a single tactic or multiple tactics joined by newlines.
Do NOT include sorry.

Example:
["ring", "norm_num", "omega", "simp; ring", "intro n\nomega"]

Respond with the JSON array only.
\end{lstlisting}

\subsection{Tactician}
\label{app:tactician_prompt}

\begin{lstlisting}
SYSTEM:
You are an expert Lean 4 theorem prover using Mathlib.
Rules:
- Lean 4 syntax only. Never output sorry.
- You continue a proof from the CURRENT GOAL that Lean
  reports. Do not restate the verified steps.
- Lean's error messages are exact. "unknown constant/
  identifier X": X does not exist, use one of the suggested
  existing names or avoid named lemmas. "type mismatch": fix
  the term or the casts. "unsolved goals": the tactic ran but
  did not finish; continue from the goal shown. "linarith/
  nlinarith/omega failed": supply the missing facts as hints
  or split the argument into `have` steps. "No goals": the
  proof was already finished; drop the extra tactic.
- Prefer short robust tactics (norm_num, simp, omega,
  linarith, nlinarith with hints, positivity, field_simp,
  ring, decide, interval_cases, induction, rcases/obtain).

USER:
## Theorem
{theorem}
## Verified so far (accepted by Lean)
{prefix}
## Current goal (reported by Lean)
{goal}
## Rejected attempts from this state (do not repeat; use
## the errors)
- attempt:
    {attempt}
  failed at: `{failing line}`
  Lean: {error message}
- ...
## Lemmas that exist in Mathlib
- {name} : {statement}
- `{unknown name}` does NOT exist. Existing names close
  to it: {name} : {statement}; ...
## Strategy hint
{strategy}

Write {k} different ways to FINISH the proof from the
current goal. Each must be the remaining tactic block only.
Use genuinely different approaches.
Output ONLY a JSON array of {k} strings. Inside a string use
\n between lines and keep Lean's indentation for nested
blocks.
\end{lstlisting}

\subsection{Critic: Ranking States}
\label{app:critic_prompt}

\begin{lstlisting}
A Lean 4 proof search has reached the partial states below
for this theorem.

{theorem}

[1] verified steps:
{prefix}
remaining goal:
{goal}
[2] ...

For each state give the probability (0 to 1) that it can be
completed in a few more steps. Give 0 if the remaining goal
is false or contradictory hypotheses are needed, and a low
value if it is no simpler than the original theorem. Output
ONLY a JSON array of {n} numbers.
\end{lstlisting}

\subsection{Critic: Triage}

\begin{lstlisting}
A Lean 4 / Mathlib theorem prover (an LLM that proposes
proofs, sees Lean's errors and retries about ten times)
failed its first attempts at this theorem.

{theorem}

First attempts and Lean's verdicts:
- {attempt}
  => {error message}
- ...

Judge the theorem exactly as Lean parses it (natural-number
subtraction and division, casts, operator precedence,
missing hypotheses). Answer with ONLY a JSON object:
{"false_as_stated": true or false,
 "p_solved": probability between 0 and 1 that the prover
             solves it with its remaining attempts}
\end{lstlisting}

\subsection{Strategist}
\label{app:strategist_prompt}

The Strategist is rule-based: it inspects the goal's syntax and returns
\texttt{direct} for goals such as \texttt{= 0}/\texttt{= 1},
\texttt{induction} for goals over \texttt{Nat} or \texttt{List},
\texttt{cases} for disjunctions, \texttt{contradiction} for negations,
\texttt{rewrite} for other equalities, and \texttt{apply\_lemma}
otherwise.  Its output is the \texttt{Strategy hint} field above.

\section{Retriever Implementation}
\label{app:retriever}

The Retriever indexes every \texttt{theorem} and \texttt{lemma}
declaration in the pinned Mathlib checkout (176{,}873 declarations,
with their fully qualified names and statements) and ranks them with
BM25 ($k_1{=}1.2$, $b{=}0.75$).  Identifiers are split on dots,
underscores and case boundaries, and operator symbols are mapped to
words, so that a goal mentioning \texttt{Nat.choose} matches
\texttt{Nat.choose\_symm}.  It is queried in two ways: with the current
goal (top four declarations), and, for each name that a Lean error among
the rejected attempts of the state reports as an unknown constant or
identifier, with that name split into words (top three declarations),
which the prompt presents as existing names close to the one that does
not exist.  In a logged development run on ten validation theorems, the
Tactician used one of the suggested names in 23 of the 32
cases that followed an unknown-name error.

\section{Batched Lean Validation and Automation}
\label{app:batched}

This appendix describes how a batch of candidates is compiled into one
Lean file, what is read from Lean's output, and the fixed tactic
combinations used by the automation.

\paragraph{Construction.}
For a state with verified prefix $\tau_{1:t}$ and candidates
$\{\pi^1,\ldots,\pi^K\}$, the harness emits a single \texttt{.lean}
file containing $K$ private copies of the theorem, each with the prefix
followed by one candidate:

\begin{lstlisting}
import Mathlib

private theorem _p0 ({hyps}) : {goal} := by
  {prefix}
  {candidate_0}

private theorem _p1 ({hyps}) : {goal} := by
  {prefix}
  {candidate_1}
...
\end{lstlisting}

\paragraph{What is extracted.}
A single \texttt{lake env lean --json} call returns one JSON line per
Lean message with its position.  The harness maps each message to a
candidate and to a line inside it.  A candidate with no message is
accepted.  A candidate whose only message is \texttt{unsolved goals} at
the level of the declaration ran without error; its remaining goal is
read from that message.  Any other candidate is rejected, and the first
error line determines the failing step.  Candidates do not interfere
because each is its own private theorem.

\paragraph{Speed-up.}
Lean's per-file overhead is paid once instead of $K$ times.  On 20
miniF2F theorems, checking $K{=}6$ candidates in one call was $5.9\times$
faster than six separate calls, and $15.3\times$ faster at $K{=}16$.

\paragraph{Closing tactics.}
At every new state, including the root, two combinations are tried in
the same Lean call that determines the state's goal:
\texttt{first | rfl | norm\_num | omega | linarith | simp | ring |
positivity | decide} and \texttt{first | nlinarith | simp\_all | aesop |
field\_simp | tauto}.

\paragraph{Step repair.}
For a rejected candidate, the step at the first error line (with its
nested lines) is replaced by \texttt{first | done | rfl | omega |
linarith | positivity | (norm\_num; done) | nlinarith | grind | bound |
(simp; done) | simp\_all | aesop | (field\_simp; ring) | (push\_cast;
ring) | (ring\_nf; done) | gcongr | decide | tauto | (norm\_num at *;
done) | (simp at *; done)}.  If the message is \texttt{unsolved goals}
on a \texttt{by} block, the combination is appended to that block; if it
is \texttt{No goals}, the extra step is dropped; structural steps
(\texttt{intro}, \texttt{cases}, \texttt{induction}, \ldots) are not
replaced.  Repair is applied for up to two rounds, so that a proof with
two failing steps can still be completed.

\section{Statistical Methodology}
\label{app:stats}

All comparisons are paired on the theorems the shared Phase~1 leaves,
with one run per method.  The difference between two methods is the
difference in the number of these theorems they solve; intervals and
$p$-values come from a bootstrap over theorems with 20{,}000 resamples.
\Cref{tab:ci} gives the 95\% intervals for the test split.  They
describe uncertainty over benchmark items and do not capture run-to-run
variation, for which \cref{tab:runs} gives an indication.

\begin{table}[ht]
\centering
\scriptsize
\setlength{\tabcolsep}{3pt}
\begin{tabular}{l l l}
\toprule
Compared with VERITAS & Haiku 5.5 & GPT-6 Luna \\
\midrule
One expansion          & $[-14,-2]$  & $[-21,-5]$ \\
$-$ Lean error feedback & $[-8,-1]$  & $[-11,+1]$ \\
$-$ Critic             & $[-4,+6]$   & $[0,+10]$ \\
$-$ tree               & $[-7,+3]$   & $[0,+5]$ \\
$-$ automation         & $[-5,+3]$   & $[-7,+5]$ \\
$-$ Retriever          & $[-5,+3]$   & $[-6,+4]$ \\
$-$ Strategist         & $[-4,+6]$   & $[-4,+6]$ \\
COPRA-style            & $[-8,+2]$   & $[-15,-3]$ \\
Reflexion-style        & $[-13,-2]$  & $[-24,-8]$ \\
DSP+-style             & $[-43,-26]$ & $[-59,-38]$ \\
\bottomrule
\end{tabular}
\caption{\textbf{95\% bootstrap intervals} of the difference to VERITAS
in solved theorems (miniF2F test).}
\label{tab:ci}
\end{table}

\section{Retries and Repeated Runs}
\label{app:runs}

\Cref{tab:runs} gives, for every run, the number of solved theorems
before and after the retry of theorems that reached the time limit, and
the number of theorems retried.  The retry changes no count by more than
five theorems and does not change the order of the methods.  The table
also lists two further VERITAS runs per model.  Of the 69 Phase~2
theorems of Claude Haiku~5.5, 34 are solved in all three runs, 4 in two,
4 in one and 27 in none; of the 95 of GPT-6 Luna, 49, 4, 7 and 35.

\begin{table}[ht]
\centering
\scriptsize
\setlength{\tabcolsep}{3pt}
\begin{tabular}{l rrr rrr}
\toprule
 & \multicolumn{3}{c}{Haiku 5.5} & \multicolumn{3}{c}{GPT-6 Luna} \\
\cmidrule(lr){2-4} \cmidrule(lr){5-7}
Run & first & final & retried & first & final & retried \\
\midrule
VERITAS                 & 212 & 212 & 5 & 199 & 202 & 12 \\
VERITAS, run 2          & 212 & 214 & 5 & 200 & 203 & 14 \\
VERITAS, run 3          & 212 & 213 & 4 & 201 & 204 & 13 \\
One expansion           & 204 & 204 & 1 & 189 & 189 & 1 \\
$-$ Lean error feedback & 208 & 208 & 2 & 195 & 197 & 17 \\
$-$ Critic              & 213 & 213 & 6 & 205 & 207 & 16 \\
$-$ tree                & 209 & 210 & 2 & 201 & 204 & 14 \\
$-$ automation          & 211 & 211 & 1 & 198 & 201 & 14 \\
$-$ Retriever           & 211 & 211 & 4 & 197 & 201 & 13 \\
$-$ Strategist          & 213 & 213 & 3 & 203 & 203 & 10 \\
COPRA-style             & 208 & 209 & 8 & 188 & 193 & 53 \\
Reflexion-style         & 205 & 205 & 5 & 184 & 186 & 23 \\
DSP+-style              & 177 & 177 & 0 & 153 & 153 & 21 \\
\bottomrule
\end{tabular}
\caption{\textbf{Solved test theorems before (\emph{first}) and after (\emph{final}) the retry}, and theorems retried.}
\label{tab:runs}
\end{table}

\section{Theorems Solved by One Method Only}
\label{app:only}

\begingroup\raggedright

Among the theorems Phase~1 leaves, the following are solved by VERITAS
and by none of the three baselines.  Claude Haiku~5.5:
\texttt{aime\_\allowbreak 1999\_\allowbreak p11}, \texttt{amc12a\_\allowbreak 2021\_\allowbreak p14}, \texttt{imo\_\allowbreak 1982\_\allowbreak p1}.  GPT-6 Luna:
\texttt{algebra\_\allowbreak sum1onsqrt2to1onsqrt10000lt198}, \texttt{amc12a\_\allowbreak 2008\_\allowbreak p25}, \texttt{amc12b\_\allowbreak 2021\_\allowbreak p3}, \texttt{imo\_\allowbreak 1963\_\allowbreak p5}, \texttt{imo\_\allowbreak 1977\_\allowbreak p6}, \texttt{induction\_\allowbreak pprime\_\allowbreak pdvdapowpma}, \texttt{mathd\_\allowbreak numbertheory\_\allowbreak 764}, \texttt{numbertheory\_\allowbreak 3pow2pownm1mod2pownp3eq2pownp2}, \texttt{numbertheory\_\allowbreak fxeq4powxp6powxp9powx\_\allowbreak f2powmdvdf2pown}.

Solved by a baseline and not by VERITAS.  Claude Haiku~5.5:
\texttt{algebra\_\allowbreak apbmpcneq0\_\allowbreak aeq0anbeq0anceq0}, \texttt{amc12\_\allowbreak 2001\_\allowbreak p21}, \texttt{amc12a\_\allowbreak 2021\_\allowbreak p19}, \texttt{imo\_\allowbreak 1992\_\allowbreak p1}.  GPT-6 Luna:
\texttt{amc12\_\allowbreak 2000\_\allowbreak p1}, \texttt{amc12a\_\allowbreak 2021\_\allowbreak p14}, \texttt{imo\_\allowbreak 1974\_\allowbreak p3}, \texttt{mathd\_\allowbreak numbertheory\_\allowbreak 451}.
\par\endgroup

\section{A Further Example}
\label{app:examples}

\Cref{fig:examples2} gives the proof of Figure~\ref{fig:teaser} in full
(Claude Haiku~5.5, first expansion).  Four of the five Phase~1 attempts
use a lemma name that does not exist in Mathlib.  The two steps before
it are kept as a verified partial proof; for each unknown name the
Retriever lists existing names close to it, and the Tactician's proof
with one of them is accepted.

\begin{figure}[ht]
\begin{lstlisting}[emptylines=1,frame=single,framesep=3pt,rulecolor=\color{black!35},xleftmargin=4pt,xrightmargin=4pt]
theorem induction_pprime_pdvdapowpma (p a : ℕ)
  (h₀ : 0 < a) (h₁ : Nat.Prime p) : p ∣ a ^ p - a

-- Phase 1 attempt: rejected at its last step
haveI := Fact.mk h₁
have hle : a ≤ a ^ p :=
  Nat.le_self_pow h₁.ne_zero a
rw [← ZMod.natCast_eq_zero_iff_dvd,
    Nat.cast_sub hle]
simp [ZMod.pow_card]
-- Lean: Unknown constant
--   `ZMod.natCast_eq_zero_iff_dvd`

-- Phase 2, from the state after `hle`: accepted
rw [← ZMod.natCast_eq_zero_iff, Nat.cast_sub hle,
    Nat.cast_pow, ZMod.pow_card, sub_self]
\end{lstlisting}
\caption{\textbf{The proof of Figure~\ref{fig:teaser}} (logged attempt and Lean message).}
\label{fig:examples2}
\end{figure}

\section{CombiBench Evaluation}
\label{app:combibench}

CombiBench~\citep{liu2025combibench} complements miniF2F with
combinatorics problems, which are rare in other Lean benchmarks.  Its
authors formalized 100 problems in Lean~4, from school level to the
IMO.  We describe the subset we use and the results of the evaluated
system.

\paragraph{Subset.}
We use the 55 problems that state a theorem to prove and leave out the
fill-in-the-blank problems, where the answer has to be found before the
proof.  The statements, with their preambles of definitions (e.g.\
\texttt{IsMagicSquare}), are taken unchanged from the public release that was available when
we ran the experiments.
The 55 theorems come from three groups of sources.

\smallskip\noindent
\begin{tabular}{@{}l r@{}}
IMO~/~IMO Shortlist & 22 \\
Brualdi, \emph{Intro.\ Combinatorics} (5th ed.) & 29 \\
APMO, Baltic Way, IZhO & ~4 \\
\end{tabular}

\paragraph{Results.}
We ran the evaluated system once with Claude Haiku~5.5 and the settings
of \cref{tab:hyperparams}.  Best-of-5 sampling solves 10 of the 55
problems, all of them Brualdi exercises.  VERITAS solves 17, seven of
them in Phase~2 and most of them Brualdi exercises.

\paragraph{License.}
CombiBench is distributed by its authors under the MIT license.  Our
repository includes the statements with this license and the list of
the 55 problems we use.

\section{Future Work}
\label{app:future}

The results point to four directions that build on what VERITAS
already records.

\paragraph{Planning longer arguments.}
The theorems that remain need chains of intermediate facts.  A
Strategist that proposes a sequence of lemmas, each checked by Lean
before the next is attempted, would extend the verified partial proofs
from single steps to whole plans.

\paragraph{Finding the answer first.}
The CombiBench problems we left out ask for a value before its proof.
The search could propose candidate answers and try to prove each one,
and Lean's messages on a failed proof would then also indicate when an
answer is wrong.

\paragraph{A learned Critic.}
Every run logs states, Lean's messages and outcomes.  Training the
Critic on these records could replace its zero-shot estimates and steer
the search towards states that lead to proofs.

\paragraph{Beyond Lean and beyond one model.}
The same loop applies wherever a checker explains its verdict, for
example in Coq, Isabelle or SMT-based program verification, and mixing
models within one search, with each model proposing candidates for the
same state, is a direct extension.

\end{document}